\documentclass[one column, journal, 11 pt]{ieeeconf}

\IEEEoverridecommandlockouts                              

\usepackage[utf8]{inputenc}
\usepackage{cite}
\usepackage{amsmath,amssymb,amsfonts}
\usepackage{steinmetz}
\usepackage{algorithmicx}
\usepackage{algorithm}
\usepackage{algpseudocode}
\usepackage{graphicx}
\usepackage{amsmath,amssymb}
\usepackage{textcomp}
\usepackage{algpseudocode} 
\usepackage{mathtools}
\usepackage[dvipsnames]{xcolor}
\usepackage{nicefrac}
\usepackage{txfonts}
\usepackage{epstopdf}

\title{\LARGE\bf Target Defense against Sequentially Arriving Intruders}

\author{Arman Pourghorban, Michael Dorothy, Daigo Shishika, Alexander Von Moll, and Dipankar Maity
\thanks{A.~Pourghorban and D. Maity  are with the Department of Electrical and Computer Engineering, at The University of North Carolina at Charlotte, NC, 28223, USA. 
Emails: 
{\tt apourgho@uncc.edu, dmaity@uncc.edu}
}%
\thanks{
M. Dorothy is with the Computational and Information Sciences Directorate, Army Research Laboratory, APG, MD, 20783, USA.
Email: 
{\tt michael.r.dorothy.civ@mail.mil}
}
\thanks{
D. Shishika is with the Mechanical Engineering Department, at George Mason University,  Fairfax, VA, 22030,  USA.
Email:
{\tt dshishik@gmu.edu}
}
\thanks{
A. Von Moll is with Control Science Center, Air Force Research Laboratory, 
WPAFB, OH, 45433, USA.
Email: {\tt alexander.von\_moll@us.af.mil}
}
\thanks{
This research is supported by the ARL grant ARL DCIST CRA W911NF-17-2-0181.
The views expressed in this paper are those of the
authors and do not reflect the official policy or position of
the United States Air Force, United States
Army,  Department of Defense, United States Government, or
its components.
}
}

\newcommand{\re}{\color{red}}

\DeclareMathOperator{\EX}{\mathbb{E}}
\newcommand{\rt}{\rho_{_T}}
\newcommand{\ra}{\rho_{_A}}
\newcommand{\xd}{\mathbf{x}_{D}}
\newcommand{\xa}{\mathbf{x}_{A}}
\newcommand{\xc}{\mathbf{x}_{C}}
\newcommand{\thetae}{\theta_{\rm eng}}
\newcommand{\taue}{\tau_{\rm eng}}
\newcommand{\phie}{\phi_{\rm eng}}
\newcommand{\reng}{r_{\rm eng}}
\newcommand{\Se}{S_{\rm engage}}
\newcommand{\R}{\mathbb{R}^2}
\newcommand{\uvec}{\hat{\mathbf{u}}}
\newcommand{\x}{\textbf{x}}
\newcommand{\ro}{r_{_T}}
\newtheorem{lemma}{Lemma}
\newtheorem{corr}{Corollary}
\newtheorem{definition}{Definition}
\newtheorem{theorem}{Theorem}
\newtheorem{prop}{Proposition}
\newtheorem{remark}{Remark}
\begin{document}

\maketitle
\thispagestyle{empty}
\pagestyle{empty}
\begin{abstract}
We consider a variant of the target defense problem where a single defender is tasked to capture a sequence of incoming intruders.
The intruders' objective is to breach the target boundary without being captured by the defender. 
As soon as the current intruder breaches the target or gets captured by the defender, the next intruder appears at a random location on a fixed circle surrounding the target. 
Therefore, the defender's final location at the end of the current game becomes its initial location for the next game.
Thus, the players pick strategies that are advantageous for the current as well as for the future games. 
Depending on the information available to the players, each game is divided into two phases: \textit{partial information} and \textit{full information phase}.
Under some assumptions on the sensing and speed capabilities, we analyze the agents' strategies in both phases.
 We derive equilibrium strategies for both the players to optimize the capture percentage using the notions of \textit{engagement surface} and  \textit{capture circle}. 
We quantify the percentage of capture for both finite and infinite sequences of incoming intruders. \\

\end{abstract}


\section{Introduction}


Considering the increase of usage and applications of robots in securing and supervising regions, considerable research has been
focused on the problem of guarding a target.
%
%
%
Isaacs \cite{isaacs1999differential} first analyzed a two-player differential game in which the \textit{pursuer} is tasked to chase an \textit{evader}.
Since his pioneering work, several extensions to pursuit-evasion games (PEGs) have been proposed, which incorporate real-world constraints such as obstacles \cite{oyler2016pursuit},  sensing limitations \cite{4660318} and visibility constraints \cite{bhattacharya2009existence}.
Similarly, different variations of PEGs have also been investigated in the literature.
Among these, the Target-Attacker-Defender (TAD) games \cite{liang2019differential,garcia2020optimal} and the reach-avoid games \cite{9091321,chen2014path,9550127} are of particular interest for our work.  
In the TAD game framework, the target and the defender form a team to guard the former from being intercepted by an attacker/intruder; whereas the attacker tries to intercept the target before getting intercepted/captured by the defender.
In reach-avoid games, the attackers attempt to \textit{reach}  certain regions while \textit{avoiding} the defenders that try to prevent them from reaching their designated areas. 
 
Perimeter defence games are a variant of the reach-avoid and TAD games.
In this variation, a group of intruders is tasked to breach a target and the defenders' job is to capture these intruders before they reach the target perimeter.
In some cases, the defenders' motion is constrained on the target perimeter \cite{shishika2020cooperative,9030082} and in other cases they move freely in the environment \cite{9561995}.
The percentage of captured intruders is of particular interest in this type of games, whereas in other variations of PEGs, the capture distance and/or the fuel consumption is of primary interest \cite{sarkar2021finite}. 


In this paper, we consider the problem of guarding a circular target region from incoming intruders. 
All the agents (target, intruders, and defender)  have limited sensing regions. 
Therefore, they may not have information about their opponents all the time.
This sensing limitation enforces them to consider both positional and sensing vantage points while choosing their strategies.
The intruders appear randomly at a fixed distance from the target center (see Fig.~\ref{fig:introPic}).
We consider a sequential arrival of the intruders in the sense that, at any time, exactly one intruder tries to breach the target boundary.
Once the current intruder succeeds or gets captured, the next intruder appears. 
This is a natural extension to the work in \cite{9561995} where the target had to be defended from a \textit{single} intruder instead of a sequence of intruders.
Similar problems have been studied in \cite{adler2022role, bajaj2021competitive, macharet2020adaptive, bajaj2019dynamic} where multiple intruders may attempt to breach the target at the same time.
However, in these works, the strategies for the intruders are simple and \textit{fixed} a priori. 
More specifically, these works assume that the intruders come directly at the target and they do not try to evade while being captured. 
Furthermore, these works do not consider sensing capability for the intruders.
In contrast to these works, we show that the sensing capability results in a particular type of evasive maneuver where the intruders can \textit{force} the defender to pursue them and thus the defender ends up  capturing the intruders at locations which are advantageous for the next intruders to increase their target breaching probability.
 
 Due to the sequential arrival of the intruders, our problem becomes a sequence of one-vs-one games where the next game starts as soon as the current game ends.
We study this problem by decomposing each one-vs-one game into two phases: \textit{partial} and \textit{full information phases}.
Each one-vs-one game starts with the partial information phase and then either it proceeds to the full information phase or the game is terminated in the breach of the target. 
The objective of the defender (intruders) is to maximize (minimize) the percentage  captured intruders.

The main contributions of this paper are: 
(i) Formulating and  solving a sensing limited target defense game against sequentially incoming intruders {{where the intruders are equipped with a limited range sensor through which they can detect the defender from a certain distance and take evasive maneuver}},
(ii) Analyzing the equilibrium strategies for the agents by defining and using the concepts of \textit{engagement surface} and \textit{capture circle} {where the former provides the favourable configurations to engage with an intruder and the latter is the set of all possible capture locations. 
Under the equilibrium strategies, the set of all possible capture locations become a
 a circle with fixed radius and center},
(iii) Analytically computing the capture percentage under both finite and infinite number of intruder arrivals. 
(iv) Numerically validating (using Monte-Carlo type random trials of experiments) the theoretically found capture probability, and finally,
(v) Characterizing the relationship between the capture probability and the problem parameters (e.g., sensing radius, speed ratio etc.) through simulations.

The rest of the paper is organized as follows:
In Section~\ref{sec:ProbFormulation}, we formulate our problem, provide the assumptions for this work, and discuss some useful definitions and necessary background materials. 
The two phases (\textit{partial} and \textit{full information}) of the one-vs-one games are discussed in Sections~\ref{sec:partInfo} and \ref{sec:fullInfo}, respectively. 
We analyze the whole game for the entire sequence of arrivals in Section~\ref{sec:GameAnalysis} and describe the progression of the game in Algorithm~\ref{algo}.
In this section we also derive the capture percentage. 
Numerical results are discussed  in Section~\ref{sec:Simu}.
We conclude the paper in Section~\ref{sec:Conclusion}. \\

\textit{Notation:} 
All vectors are denoted with lowercase bold symbols, e.g., $\mathbf{x}$. 
We use $\uvec(\theta)$ to denote the unit vector $[\cos\theta,~\sin\theta]^\intercal$.

\section{Problem Formulation} \label{sec:ProbFormulation}
\begin{figure}
    \centering
    \includegraphics[trim = 100 30 100 20, clip, width = 0.25 \textwidth]{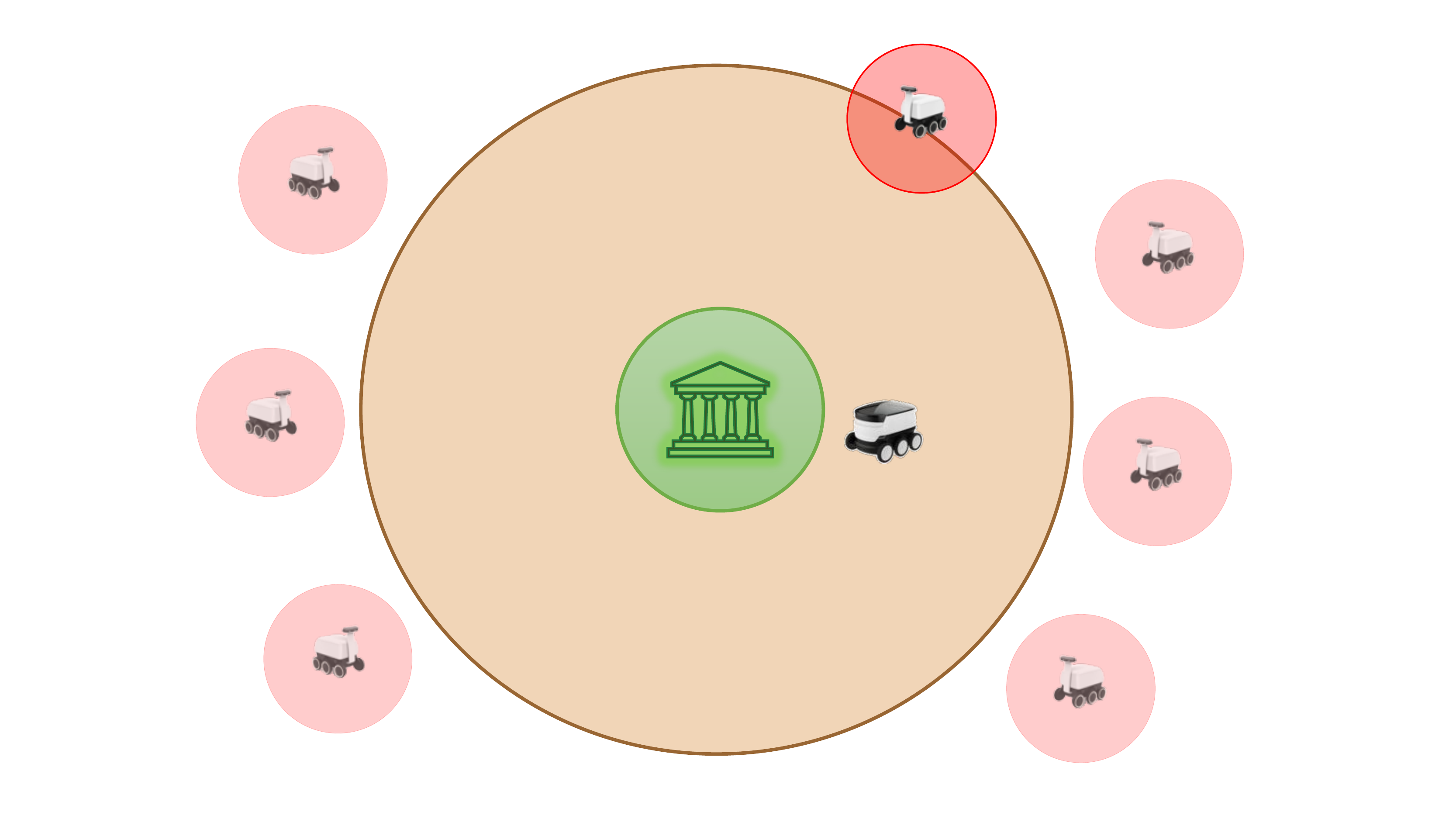}
    \put (-103, 40 ) {\line(1,0) {28}}
    \put (-90, 33) {$\rt$}
    \caption{\colorbox{green!30}{\textcolor{green!30}{gh}}: Target region with radius $\ro$, \colorbox{red!20}{\textcolor{red!20}{gh}}: intruders sensing region with radius $\ra$, \colorbox{orange!60!black!25}{\textcolor{orange!60!black!25}{gh}}: Target's sensing anulus with radius $\rt$.}
    \label{fig:introPic}
    \vspace{-6 pt}
\end{figure}
We consider a target guarding problem in ${\mathbb{R}^2}$ where a defender is tasked to protect a circular target region $\mathcal{R}_{T}=\{\x \in {\mathbb{R}^2}\ | \ \|\x\| \leq r_{_T}\}$ from an incoming sequence of intruders as schematically shown in Fig.~\ref{fig:introPic}.
The target has a sensing annulus of radius $\rt$ through which it can detect the presence of an intruder within this annulus region.
We will refer to this annulus as \textit{Target Sensing Region} (TSR).
The intruders appear on the TSR boundary sequentially.
That is, the next intruder appears on the TSR boundary as soon as the current one gets captured by the defender or breaches the target.
Typically in these types of games, the intruders are able to \textit{escape} from the TSR and the defender is unable to pursue those intruders afterwards \cite{9561995}. 
To prevent this scenario, we will consider a large enough TSR (see our parametric assumptions in \eqref{eq:Asm}) so that an intruder is not able to escape when it does not have a strategy to breach the target.
%
%
%
Therefore, each one-vs-one game terminates either in \textit{breach} of the target or in \textit{capture} of the intruder.
Each intruder appears at the TSR boundary randomly with a \textit{uniform probability} that is \textit{independent} of the previous arrivals. 
Due to the random arrivals, the number of capture (equivalently, the percentage of capture) is a random variable.
The objective of this work is to find a strategy for the defender that maximizes the expected percentage of capture. 
We will consider both the cases of finite and infinite sequences of arrivals and analyze the expected percentage of capture for both cases.

Let $\xa(t),\xd(t) \in \R$ denote the positions of the intruder and the defender at time $t$. 
Both the defender and  intruder are assumed to have first-order dynamics, i.e., 
\begin{align}
    \dot{\mathbf{x}}_A = v_A \uvec(\psi_A), \qquad \dot{\mathbf{x}}_D = v_D \uvec(\psi_D),
\end{align}
where the defender (intruder) directly controls its speed and heading angle by selecting $v_D$ and $\psi_D$ ($v_A$ and $\psi_A$), respectively.
In this problem, we assume that the defender and the intruder
have the speed limit of 1 and $\nu$, respectively, i.e., $|v_D(t)|\le 1$ and $|v_A(t)|\le \nu$ for all $t$. 
Furthermore, we assume that $\nu < 1 $, i.e., the defender is faster.
Without the constraint $\nu< 1$, it is impossible for the defender to prevent \textit{any} breaching and the percentage of capture will be zero.
Thus the outcome of the game is trivial for $\nu \ge 1$.

After arriving at the TSR boundary, the intruder moves radially toward the target center with maximum speed until it senses the defender.
The target and the defender work as a team and whenever an intruder is within the TSR, the defender has access to the intruder's instantaneous position $\xa(t)$.
The intruder is also equipped with its own sensor and  is able to sense the defender only if they are within a distance of $\ra$ or less.
Using this sensing capability, the intruder is able to find the right breaching point on the target, or able to get out of the TSR uncaptured, or is able to evade for some time before getting captured by the defender. 
 This evasive maneuver is an important capability for the intruder because this forces the defender to pursue and capture the intruder at a location that is unfavorable for the defender to start the next game.
 Thus, while getting captured, each intruder can maximize the chances of winning for the next intruder, which would not have been possible if $\ra =0$.
  %

\subsection{ Parametric Assumptions}
\par
Agent's strategies and the overall outcome of
this game depend on the game
parameters $\ro,\rho_A $, $\rho_T $ and $\nu $. In this paper we focus on a specific parameter regime
given by the following condition
\begin{align}
    \max\{(1+\nicefrac{2\nu}{(1 - \nu^2)}) \ra, ~~~ \nu \ro+\nicefrac{2\ra\nu^2}{(1 - \nu^2)} \} \le \rt
    \label{eq:Asm}
\end{align}
The first condition (i.e., $(1+\nicefrac{2\nu}{(1 - \nu^2)}) \ra\le \rt$) is necessary to ensure that there exists a strategy for the defender to capture the intruder inside the TSR. 
Otherwise, the intruder can \textit{evade} outside the TSR where the defender cannot detect it.
Thus, there can be a deadlock situation where every time the intruder cannot breach it gets out of the TSR and keeps trying indefinitely (and thus, not allowing any other intruders to appear).
The second condition {{(i.e., $\nu \ro+\nicefrac{2\ra\nu^2}{(1 - \nu^2)} \le \rt )$ is to ensure that if the defender does not have a strategy to capture the intruder and loses the current game, then it has enough time to reach the target center from the \textit{capture circle} (to be defined later) before this intruder breaches the target boundary and the next game starts.}} 
The need for these assumptions will be apparent from Remark~\ref{rem:assumption}.

\subsection{Apollonius Circle and its Importance}
Given the locations $\xa(t)$ and $\xd(t)$ of the agents  at time $t$, one can construct the locus of the points $\x$ such that the ratio of the distances from $\x$ to $\xa(t)$ and to $\xd(t)$ is $\nu$, i.e.,
\begin{align*}
    \|\x-\xa(t)\| = \nu \|\x - \xd(t)\|.
\end{align*}
For $\nu<1$, the set of such points lie on a circle with center $\xc(t)$ and radius $r_C(t)$ given by
\begin{align} \label{eq:AC}
    \xc(t) = \alpha\xa(t) - \beta \xd(t), \quad r_C(t) = \gamma\|\xa(t) - \xd(t)\|,
\end{align}
where 
\begin{align*}
    \alpha = (1-\nu^2)^{-1}, \quad \gamma = \nu \alpha,\quad \beta = \nu\gamma.
\end{align*}

The interior of the Apollonius circle is called the intruder's \textit{dominance region}. 
That is, the intruder can reach \textit{any} point within this region before getting captured.  
Both agents can reach any point on the perimeter of the Apollonius circle at the same time.
The following result from \cite[Theorem~1]{dorothy2021one} shows that the intruder \textit{cannot exit} the Apollonius circle without getting captured. 
That is, any point outside of the Apollonius circle is the dominance region of the defender.
\begin{lemma}[\!\!\cite{dorothy2021one}] \label{lem:arXiv}
Regardless of the strategy of the intruder, the defender has a strategy to capture the intruder arbitrarily close to the Apollonius circle. \hfill $\triangle$
\end{lemma}

Therefore, the Apollonius circle is a crucial component in the analysis of our game as it uniquely determines the reachable set of the intruder before it gets captured.

\subsection{Game Phases and Guarded Arc}
\textit{Full information phase}: In this phase, both agents can sense each other, i.e, $\|\xa(t)\|<\ro+\rt$ and $\|\xa(t)-\xd(t)\|\le \ra$.

\textit{Partial/Asymmetric information phase}: In this phase only the defender sees the intruder, i.e., $\|\xa(t)\|<\ro+\rt$ and $\|\xa(t)-\xd(t)\|>\ra$.
In this phase, the intruder comes radially toward the target until it senses the defender- at which point the full information phase starts.

\begin{lemma} \label{lem:captureFarthest}
Let the defender is located at $r\uvec(\theta_D)$ with $r\le \tilde{\rt} \triangleq \rt+ \ro$. 
This defender can capture any intruder incoming at an angle of $\theta_A$ such that 
%
 \begin{align*}
   |\theta_A - \theta_D| \le \begin{cases}
   \Theta_{\rm g} \qquad &\text{  if  } r > \frac{\rt}{\nu} - \ro,\\
   \pi \qquad &\text{  otherwise,}
   \end{cases} 
\end{align*}
where 
{\small
\[
\Theta_{\rm g} = 2\cos^{-1}\!\Bigg(\!\!    \sqrt{\frac{\Big((\ro+\nu r)^2-(\tilde{\rt}-\nu\ro)^2\Big)\Big( (\tilde\rt+\nu\ro)^2-(\nu r - \ro)^2 \Big)}{16\nu^2\ro^2r\tilde \rt}}\Bigg).
\]
}

\hfill $\triangle$


\end{lemma}
\begin{proof}
A proof of this lemma is omitted due to space limitation.
\end{proof}

The above lemma shows that if the defender is located within $\nicefrac{\rt}{\nu} - \ro$ distance from the target center then it can capture \textit{any} incoming intruder.
Otherwise, it can capture only the intruders that are coming with an (absolute) angular separation of no more than $\Theta_{\rm g}$. 
We further notice that $\Theta_{\rm g}$ depends on the radial position of the defender (i.e., $r$), and the nature of this dependence is presented in the following corollary. 
%
\begin{corr} \label{corr:R_theta}
 $\Theta_{\rm g}$ is a decreasing function of $r$. $\hfill \triangle$
\end{corr} 
 \begin{proof}
A proof of this corollary is omitted due to space limitation.
\end{proof}

 Therefore, it is desirable for the defender to stay as close to the target center as possible.
 This observation will be an instrumental component in determining the strategies for the defender and the intruders.

 Next we analyze each of the game phases in detail.  

\section{Partial Information Phase} \label{sec:partInfo}
In this phase, the intruder moves radially toward the center of the target with full speed until it senses the defender and starts the \textit{full information phase}.\footnote{
In this phase, the defender has information about the intruder's location and can strategically position itself to capture this intruder whenever capture is possible.
On the other hand, if the intruder does not come radially toward the target center, it spends a longer time within the TSR without knowing which direction is beneficial to move for avoiding the defender, which only benefits the defender.
A formal proof of the last statement is beyond the scope of this paper and we simply assume that the intruders are prescribed to move radially in this phase.
}
The defender's objective in this phase is to leverage the sensing (information) advantage it has and use it to start the next phase in a favorable configuration. 
Therefore, to discuss the strategy of the defender in this phase, we need to understand what configurations are favorable for the next phase and whether such a configuration is achievable by the defender. 
This is discussed in the following.

\subsection{Engagement Surface}
Suppose the full information phase starts after a $\tau$ amount of time from the appearance of the intruder at the TSR boundary. 
Without any loss of generality, we assume that each intruder appears on the TSR boundary at time $0$ at the point $[\ro+\rt,~ 0]$.
Since  the intruder moves radially toward the target center with maximum velocity during the partial information phase, the intruder shall be at a location $\xa(\tau)= [\ro+ \rho_T -\tau \nu, \ 0]^\intercal$ at time $\tau$ when the \textit{full information phase} starts, i.e., the intruder first senses the defender. 
Let the defender be located  at a $\theta$ angle on intruder’s sensing boundary at that time, i.e.,
\begin{equation} \label{eq:xdThetaxa}
    \xd(\tau) = \xa(\tau) + \rho_{A}\uvec(\theta).
\end{equation}
%
Therefore, from \eqref{eq:AC}, the center of the Apollonius circle is  at
$\xc(\tau) = \xa(\tau) -\beta \ra\uvec(\theta)$
and the radius is $\gamma \ra$. 
To guarantee capture of the intruder, the Apollonius circle must be contained within the TSR without intersecting with interior of the target, i.e.,
\begin{subequations} \label{eq:engageConditions}
\begin{align}
&\|\xc(\tau)\| \geq \ro +\gamma \ra,\\
&\|\xc(\tau)\|\le \ro + \rt- \gamma\rho_{A}.
\end{align}
\end{subequations}
Notice that, due to properties of the Apollonius circle and Lemma~\ref{lem:arXiv}, equation \eqref{eq:engageConditions} is necessary and sufficient to ensure capture.
See Fig.~\ref{fig:engage_config} for a graphical representation. 
\begin{figure}
    \centering
    \fbox{\includegraphics[trim= 300 160 270 150, clip, width = 0.22\textwidth]{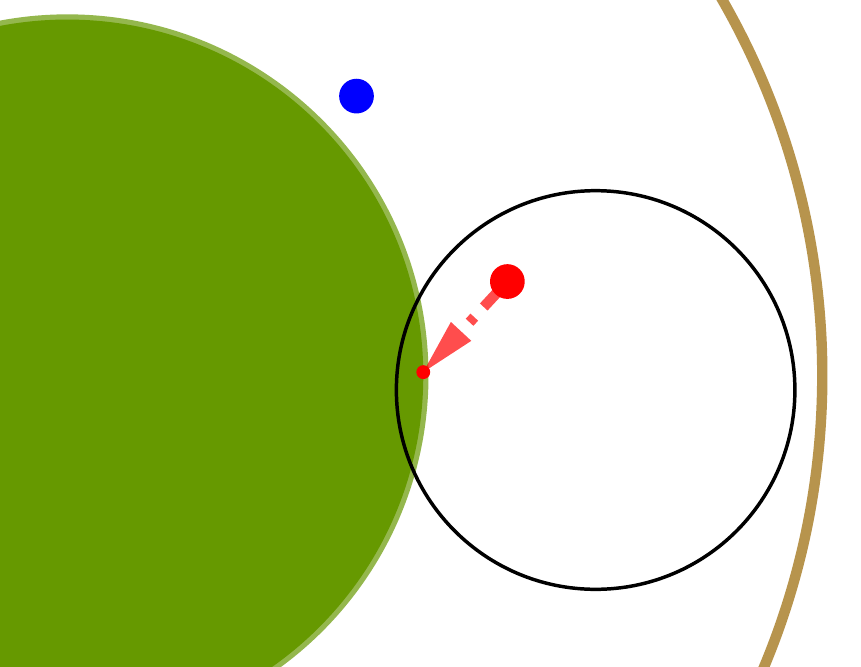}}
    \put(-50,80) {breach}
    \hspace{1 cm}
    \fbox{\includegraphics[trim= 300 160 270 150, clip, width = 0.22\textwidth]{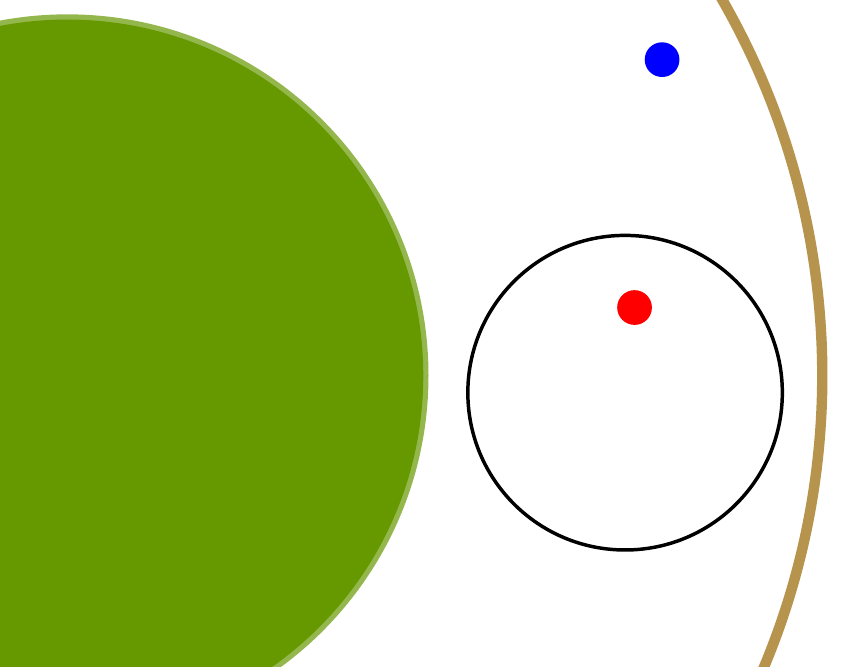}}
     \put(-70,80) {capture}
    \caption{ The blue and red dots represent the positions of the defender and the intruder, respectively. The green region is a part of the circular target and the brown arc represents a part the TSR boundary. The black circles represent the Apollonius circles. 
    In the left figure the intruder can breach the target. 
    Such a breaching location is shown by the dashed red vector. 
    The right figure shows a configuration that leads to capture. 
    }
    \label{fig:engage_config}
\end{figure}
Satisfiability of the two conditions in \eqref{eq:engageConditions} depends on the choice of $\tau$ and $\theta$ in \eqref{eq:xdThetaxa}, and there could be multiple such choices for $\tau$ and $\theta$. 
We next analyze which of these choices are going to help the defender in the current game and maximize the probability of winning the subsequent games. 

 Since the next intruder appears at the TSR boundary as soon as the current one is captured (or breaches the target), the defender needs to ensure that the current one is captured at a location that maximizes the chance of capturing the next one.
Recall from Corollary~\ref{corr:R_theta} that the \textit{capture angle} is maximized if the defender is able to capture the current one as close to the target center as possible. 
Or in other words, the Apollonius circle should be as close to the target center as possible while satisfying \eqref{eq:engageConditions}.
Therefore, $\theta$ and $\tau$ should be chosen such that the Apollonius circle is tangent to the target boundary, i.e.,
\begin{align} \label{eq:tangentCondition}
    \|\xc(\tau)\| = \ro +\gamma \rho_A.
\end{align}
Notice that, due to Assumption~\eqref{eq:Asm}, equation \eqref{eq:tangentCondition} ensures satisfaction of both the conditions in \eqref{eq:engageConditions}.
Equation \eqref{eq:tangentCondition} simplifies to
\begin{equation} \label{eq:engagementSurface}
    \sin^2(\theta/2) =   \frac{(\ro+\gamma\ra)^2 - (\ro+\rt-\tau \nu -\beta\ra)^2}{4\beta\rho_A(\ro+\rt-\tau \nu)},
\end{equation}
which provides a feasible engagement time and location pair $(\tau, \theta)$. 
The locus of these engagement points, i.e., $\xa(\taue) + \ra\uvec(\thetae)$ is plotted in Fig.~\ref{fig:engagement_surface}.

\begin{definition}[Engagement Surface]
This is the collection of the pairs $(\taue,\thetae)$ that satisfy \eqref{eq:engagementSurface}. 
That is, $\Se = \{(\thetae, \taue)~:~\thetae, \taue \text{ satisfy }\eqref{eq:engagementSurface}\}.$ \hfill$\triangle$
\end{definition}

If the \textit{full information phase} starts at a point on the \textit{engagement surface}, then \eqref{eq:tangentCondition} is satisfied and  due to Assumption~\eqref{eq:Asm} and Lemma~\ref{lem:arXiv} the intruder cannot breach or get out of the TSR before getting captured.
Therefore, the intruder's objective in this case is to get captured the farthest from the target center.
This will increase the breach probability for the next intruder due to Corollary~\ref{corr:R_theta}.
The farthest capture point will be at a radial location of $\ro + 2\gamma\ra$ since the radius of Apollonius circle is $\gamma\ra$ and the center is at a distance $\ro+\gamma\ra$ from the target center (see \eqref{eq:tangentCondition}). 
The circle with radius $\ro+2\gamma\ra$ and concentric with the target is referred to as the \textit{capture circle}, since all the captures of the intruders will happen on this circle.

\begin{figure}
    \centering
   { \includegraphics[trim = 40 20 40 20, clip, width=0.27\textwidth]{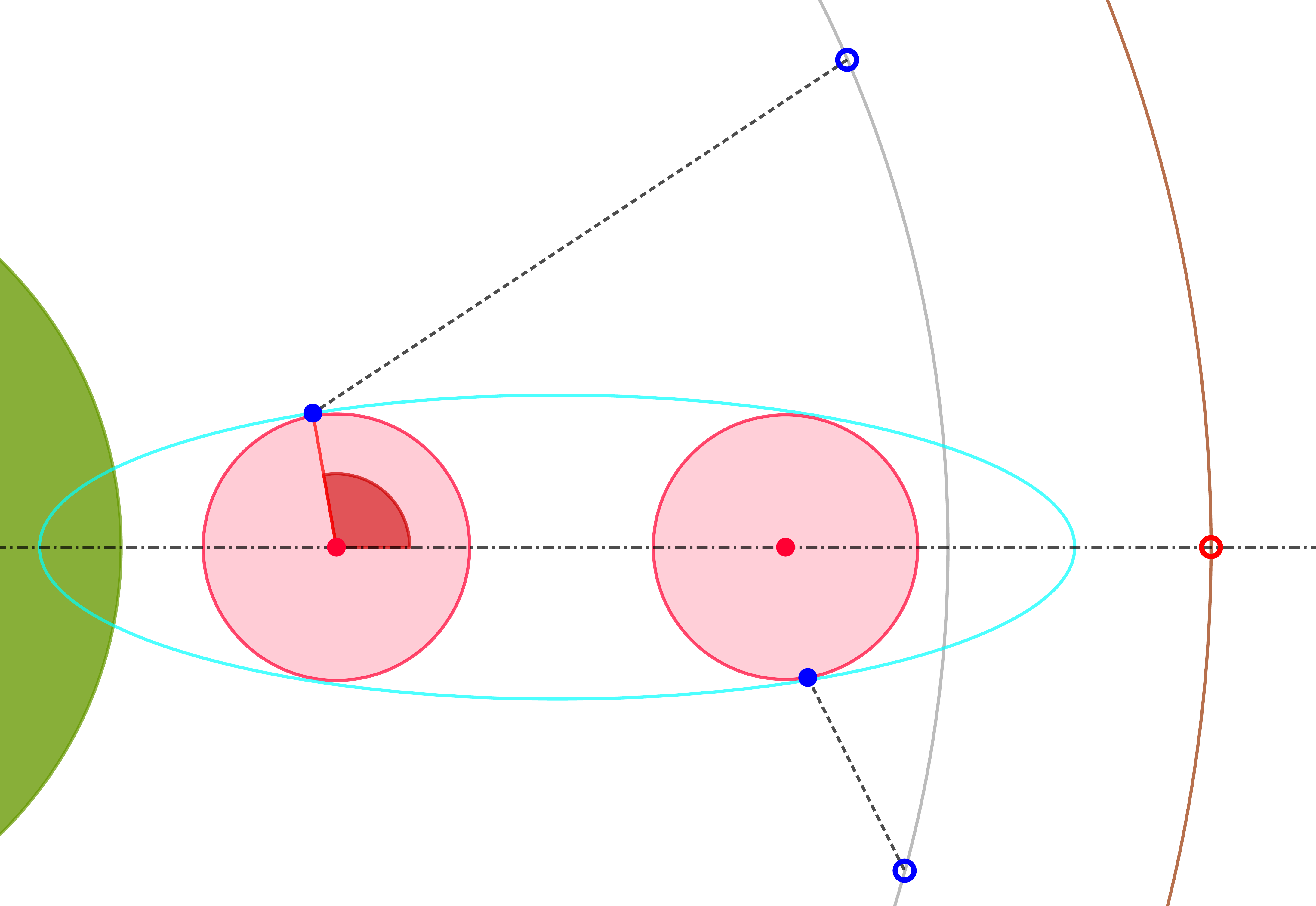}}
     \put(-108, 27) {\small \re $\xa(\taue)$}
     \put(-95, 45) {\small \re $\thetae$}
     \put(-110, 55) {\small  $x_1$}
     \put(-44, 90) {\small {$D_1$}}
      \put(-38, 0) {\small {$D_2$}}
      \put(-58, 15) {\small {$x_2$}}
    \hspace{1 cm} 
    {\includegraphics[trim = 112 120 40 20, clip, width=0.19\textwidth]{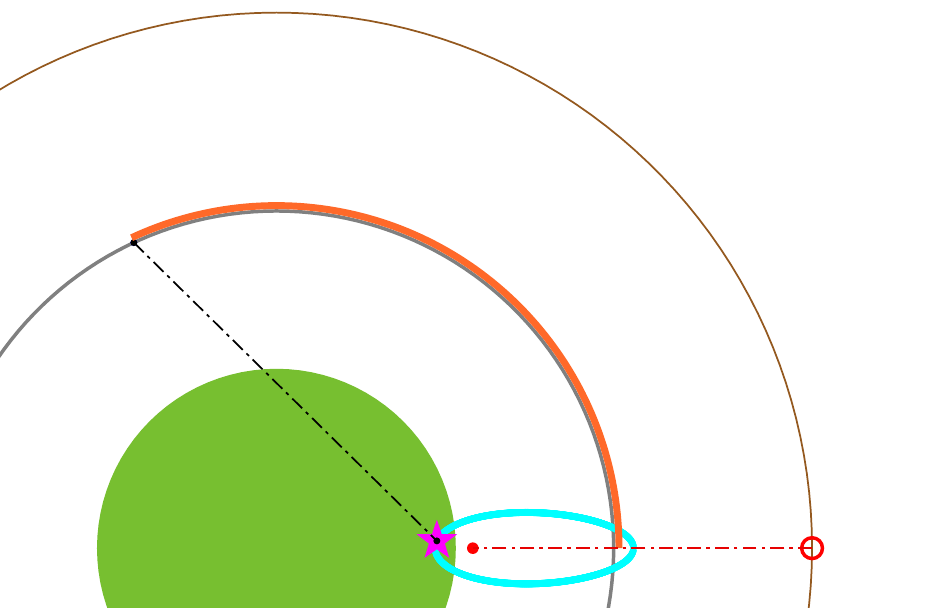}}
    \caption{ 
    (Left) The cyan oval-shaped curve represents the locus of the engagement points $\xa(\taue) + \ra\uvec(\thetae)$. 
    We demonstrate two configurations where a defender $D_1$ (or $D_2$) starts the full information phase by engaging at the point $x_1$ (or $x_2$) when the intruder is at $\xa(\taue)$.
        (Right) The zoomed out version of the left figure. 
        The outer most (brown) arc represents a part of the TSR boundary. 
         The gray arc is the \textit{capture circle}.
    The orange part of the capture circle represents all the possible locations of the defender such that it can reach the magenta star before time $\taue$.
    The magenta star represents one possible engagement point. 
    The locations of the intruder at the time of engagement and its initial location on TSR boundary are shown using  solid red dots and a hollow red circle, respectively. 
    The trajectories taken by the agents are shown using (dotted-)dashed lines.
    }
    \label{fig:engagement_surface}
    \vspace{- 10 pt}
\end{figure}


\subsection{Reachability to the Engagement Surface}
Let the defender be located at a point $r\uvec({\theta_D})$ at the time the intruder enters the TSR boundary at $[\ro +\rt, 0]^\intercal$. 
The following theorem provides the conditions on $r$ and $\theta_D$ such that the defender is able to reach the engagement surface at $(\taue, \thetae) \in \Se$.

\begin{theorem}[Necessary] \label{thm:angularSeparation}
A necessary condition for a defender located at $r\uvec({\theta_D})$ to reach $(\taue, \thetae) \in \Se$ before getting detected by the intruder is $\theta_D \le \theta_{\max}(\taue,\thetae)$, where
\begin{align}
    &\theta_{\max}(\taue,\thetae, r) = \cos^{-1} \bigg( \frac{ \reng^2 +r^2 - \taue^2}{2\reng r} \bigg)  + \phie \\
    &\phie = \sin^{-1}\bigg(\frac{\ra \sin(\thetae)}{\reng} \bigg)  \nonumber
    \end{align}
    \begin{align}
    & \reng = \Big((\ro+\rt-\taue\nu -\ra)^2 \nonumber\\
    & \qquad\quad + 4(\ro+\rt-\taue\nu) \ra\cos^2(\thetae/2)\Big)^{1/2}. \nonumber \qquad \hfill \triangle
    \end{align} 
\end{theorem}
\begin{proof}
The proof is omitted due to space limitation.
\end{proof}

Theorem~\ref{thm:angularSeparation} provides the maximum allowed initial angular separation between the intruder and the defender to ensure capture. 
Since there are many possible engagement points $(\taue,\thetae)$, we shall optimize over $(\taue,\thetae)$ to maximize $\theta_{\max}(\taue,\thetae)$.
{Without loss of generality, henceforth we will denote $(\taue, \thetae)$ to be the one that maximizes $\theta_{\max}(\taue,\thetae)$.}


Theorem~\ref{thm:angularSeparation} gives the necessary condition on the defender's initial location to ensure reachability to the engagement surface. 
However, the theorem does not provide any guarantee whether the defender will be detected by the intruder along its way to  $\x_{\rm eng}$.
For example, a scenario like the one presented in Fig.~\ref{fig:sufficient} may occur. 
In the following theorem we formally prove that a scenario like  Fig.~\ref{fig:sufficient} shall not occur under a certain condition on $r$. 
In other words, the necessary condition in Theorem~\ref{thm:angularSeparation} is also sufficient.
\begin{figure}
    \centering
    \includegraphics[trim = 50 160 128 150, clip, width = 0.45\textwidth]{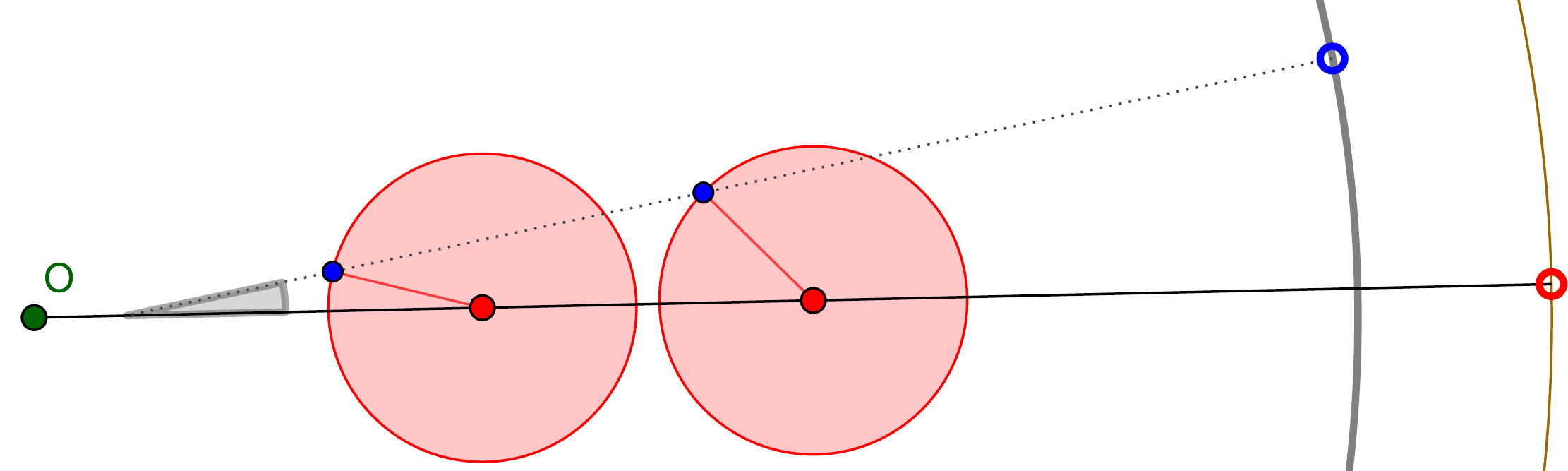}
    \put(-175,15) {$\xa(\taue)$}
    \put(-120,15) {$\xa(\tau_1)$}
    \put(-115,35) {\re$\ra$}
    \put(-30,25) {$\underset{{\ro+\rt - r}}{\underbrace{~~}}$}
    \put(-195,30) {\color{gray}$\theta$}
    \caption{ {\re $\medbullet$ } and {\color{blue} $\medbullet$ }  denote the location of the intruder and defender at times $\tau_1$ and $\taue$. The hollow red and blue circles denote their initial locations.
    The brown arc denotes a part of the TSR boundary. 
    The gray arc has a radius of $r$ and it is concentric to the target center (denoted by the green point $O$). 
    The defender moves along the dotted line. 
    The angle between the defender's path and the intruder's path is denoted by $\theta$.
    } \label{fig:sufficient}
\end{figure}
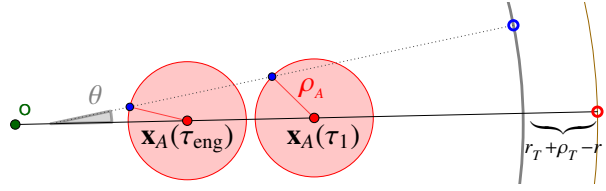
\begin{theorem}[Sufficient]\label{thm:sufficiency}
  If $r\le \ro + \rt - \ra$, then $\theta_D \le \theta_{\max}(\taue,\thetae)$ is a sufficient condition for a defender located at $r\uvec({\theta_D})$ to reach $(\taue, \thetae) \in \Se$ before getting detected by the intruder. $\hfill \triangle$
\end{theorem}



\begin{proof}
The proof is omitted due to space limitation.
\end{proof}

The following remark is due to Theorem~\ref{thm:sufficiency} and Assumption~\eqref{eq:Asm}. 
This remark is important for analyzing the entire game with multiple arrivals.

\begin{remark} \label{rem:assumption}
If a defender is located on the \textit{capture circle} at the time an intruder appears on the TSR boundary, then $r=\ro+2\gamma\ra < \ro+\rt-\ra$ due to Assumption~\eqref{eq:Asm}.
Therefore, if $\theta_D\le \theta_{\max}(\taue,\thetae)$ then due to Theorem~\ref{thm:sufficiency}, the defender is able to reach the engagement surface and the intruder will then be captured on the \textit{capture circle}.
As a consequence, the condition (on $r$) for Theorem~\ref{thm:sufficiency} will be satisfied for the next game too.
On the other hand, if $\theta_D > \theta_{\max}(\taue,\thetae)$, then  the defender will lose the current game.
Therefore, in that case, the defender goes back to the target center  to ensure that it wins the next game with probability 1. 
Due to the Assumption that $\rt/\nu > \ro + 2\gamma\ra$ in \eqref{eq:Asm}, the defender can reach the target center before the next game starts. \hfill $\triangle$
\end{remark}

Based on the above remark, we conclude this section with the following proposition on the defender's strategy during the partial information phase.
\begin{prop}[Defender Strategy]
If $\theta_D\!>\!\theta_{\max}(\taue,\thetae)$, then the optimal strategy for the defender is to go to the center of the target, otherwise, it should go to the point $(\taue,\thetae)\in \Se$ to start the full information phase. $\hfill \triangle$
\end{prop}
\begin{proof}
When $\theta_D > \theta_{\max}(\taue,$ $\thetae)$, the defender is not able to reach to \textit{engagement surface}, and therefore, at the full information phase the Apollonius circle will have an intersection with the target.
This will lead to a breach. 
Therefore, the defender's attempt to engage with the intruder is futile and at the end of this engagement the defender might be at a non-zero radial location which will decrease its probability of winning the next game (Corollary~\ref{corr:R_theta}).
Thus, an \textit{optimal} strategy for the defender is to maximize the probability of the next game, which is achieved by being closest to the target center. 

On the other hand, when $\theta_D \le \theta_{\max}(\taue,$ $\thetae)$, then the defender can reach the engagement surface or it may choose not to engage with this intruder and rather position itself better for the next intruder. 
Given that the intruders are arriving on the TSR boundary with uniform probability, it can be verified (by considering the expected number of captures) that the \textit{optimal} choice is to engage with the current intruder since it is going to be a win with probability 1 (and some strictly positive probability for capturing the next).
\end{proof}


\section{Full Information Phase} \label{sec:fullInfo}
In the previous section we discussed the favourable configurations for starting the \textit{full information} game and derived the conditions on the defender's location to ensure that such a favorable configuration is achievable.
In this section we discuss the progression of the game once the engagement starts.

As soon as the intruder detects the defender, it can compute the instantaneous Apollonius circle and verify whether the configuration leads to a breach or capture. 
Recall from Lemma~\ref{lem:arXiv} that the intruder cannot get out of this Apollonius circle without being captured. 
Therefore, if breach is possible (i.e., the Apollonius circle has a nonempty intersection with the interior of the target), then the intruder moves toward a breaching point (see left image in Fig.~\ref{fig:engage_config}).
Otherwise, the intruder picks a strategy to draw the defender farthest away from the target center. 
In this way, the intruder ensures the radial distance of the defender is maximized before the next game starts and the defender has a lesser guarding angle (see Corollary~\ref{corr:R_theta}).
The following lemma provides the strategy for an intruder that gets captured.
\begin{lemma} \label{lem:xp}
The intruder goes to the point $\x_p$ in a straight line with maximum velocity, where
\begin{align*}
   \x_p = \xa(\taue) - \beta \ra\uvec({\thetae}) + \gamma \ra \uvec( \varphi) 
\end{align*}
Where $\varphi$ is equal to:
\begin{align} \label{eq:varphi}
\varphi = \tan^{-1}\bigg( \frac{\beta\ra \sin({\thetae})}{\beta\ra \cos({\thetae})-(\ro+\rt-\taue \nu)}\bigg) 
\end{align}
$\hfill \triangle$
\end{lemma}
\begin{proof}
The full information phase starts with intruder being at $\xa(\taue)=[\ro+\rt-\taue \nu, 0]^\intercal$ and the defender being at $\xd(\taue) = \xa(\taue) + \ra\uvec({\thetae})$.
Therefore, from \eqref{eq:AC}
\begin{align*}
    \xc(\taue) = \xa(\taue) - \beta \ra\uvec({\thetae}) = (\ro+\gamma\ra) \uvec({\varphi}),
\end{align*}
where $\varphi$ is given \eqref{eq:varphi}.
Thus, the farthest point on the Apollonius circle from the target center is at
\begin{align*}
    \x_p = \xc(\taue) +  \gamma \ra \uvec( \varphi) .
\end{align*}
By substituting the expression of $\xc(\taue)$ in the above equation, the lemma is proved.
\end{proof}



\section{Analysis of the Game} \label{sec:GameAnalysis}
The first game starts with defender being at the target center. 
 Due to Assumption \eqref{eq:Asm} and Lemma~\ref{lem:captureFarthest}, this leads to a capture of the first intruder with probability 1.
 The defender is on the \textit{capture circle} at the end of the first game. 
 Afterwards, at the end of each game it will either be on the capture circle or on the target center depending on the arrival angles of the intruders.
 The defender strategy (and the progression of the game) is described by the following algorithm.
       
    

\begin{algorithm} 
    \caption{Defender's Strategy} \label{algo}
    \label{euclid}
    \begin{algorithmic}[1] 
    \State Initialize $\xd \gets [0,0]^\intercal$, $N_{\rm capture}\gets 0$,  and $N$
    \For{$n = 1: N$}
    \State $\theta_A \sim {\mathcal U}(-\pi, \pi)$ \Comment{Uniform random arrival of intruder}
    \If{$\xd$ is at target center} \Comment{{\color{blue!60}Capture happens}}
       \State ${\x}_{\rm eng} \gets (\ro - \nicefrac{\ra}{(1+\nu)})\uvec(\theta_A)$
       \State Defender goes to ${\x}_{\rm eng}$ 
       \State $N_{\rm capture} \gets N_{\rm capture} +1$
       \State $\xd \gets (\ro +2\gamma\ra)\uvec(\theta_A)$ \Comment{$\xd$ after capture}
    \ElsIf{$\xd = (\ro+2\gamma\ra)\uvec(\theta_D)$}
    \If{$|\theta_D - \theta_A|\le \theta_{\max}$} \Comment{{\color{blue!60}Capture happens}}
        \State ${\x}_{\rm eng} \gets \xa(\taue)+\ra\uvec(\thetae)$
        \State Defender goes to ${\x}_{\rm eng}$ 
        \State $N_{\rm capture} \gets N_{\rm capture} +1$
        \State $\xd \gets \x_p$ from Lemma~\ref{lem:xp} \Comment{$\xd$ after capture}
       \Else \Comment{{\color{red!60} Breach happens}}
       \State Defender goes to the target center
       \State $\xd \gets [0,0]^\intercal$
       \EndIf
    \EndIf
    \EndFor
    \end{algorithmic}
\end{algorithm}

\subsection{Percentage of capture}
If the defender is located at the target center, then capture is guaranteed regardless of the angle $\theta_A$ the intruder appears at the TSR boundary. 
Otherwise (when it is on \textit{capture circle}), it can only capture the intruders that appear on the TSR boundary with an angular separation of $\theta_{\max}$ (Theorems~\ref{thm:angularSeparation}, \ref{thm:sufficiency}).
%
%
Given that the intruders appear on the TSR boundary independently with uniform probability, the capture probability is
\begin{align} \label{eq:p*}
   p^* \triangleq  \frac{\theta_{\max}}{\pi}.
\end{align}
\subsection{Percentage of Capture with Finite Arrival of Intruders}
As discussed before, the first intruder gets captured with probability 1.
 After the first capture, the defender is on the \textit{capture circle} and can capture the next randomly appearing intruder with probability  $p^*$.
If the next intruder is captured, then the defender will remain on the \textit{capture circle}; otherwise, it will be at the target center at the time the third game starts. 
This process is repeated as the intruders keep coming.

Every time the defender goes back to the center of the target (i.e., an intruder is able to breach) we call this a \textit{reset} of the game. 
The (random) sequence of capture in-between two \textit{resets} is called a \textit{travel}. 
The length of a \textit{travel} is the number of captures in that travel. 
Let the random variable $L_i$ denote the length of the $i$-th \textit{travel}.
Since each intruder appears independently with uniform probability at the TSR boundary, one may verify that $L_i$ is a geometric random variable with 
\begin{align} \label{eq:prob_Li}
  P(L_{i}=k) = (p^*)^ {k-1} (1- p^*),
\end{align}
for all positive integer $k$,  where $p^*$ is the quantity defined in {\eqref{eq:p*}}.
The probability of capturing a total of $n$ intruders in $m$($\le n$) {travels} is therefore \begin{align*}
  P\bigg(\sum_{i=1}^{m} L_{i}=n \bigg) = \binom{n-1}{m-1} (p^*)^ {n-m} (1- p^*)^{m},
\end{align*}
 which follows a negative binomial distribution.
 
 Let $N$ denote the total number of intruders that have arrived at the TSR boundary so far and $S_N$ denote the number of \textit{resets} at the end of the $N$-th game. 
 Therefore, we have 
\begin{align*}
    N \geq \sum_{i=1}^{S_N} L_{i} + S_N
\end{align*}
with probability almost surely.
Notice that the event $\{S_N \le m\}$ is equivalent to the event $\{\sum_{i=1}^{m+1} L_{i} +m \ge N\}$. 
Therefore, 
\begin{align}
 &P(S_N > m) = 1 -P(S_N \leq m)\nonumber\\
 &= 1- P\bigg(\sum_{i=1}^{m+1} L_{i} +m \ge N \bigg) \nonumber  \\
 &= P\bigg(\sum_{i=1}^{m+1} L_{i} +m \le N-1 \bigg) = \sum_{j= m+1}^{N-m-1} P\bigg(\sum_{i=1}^{m+1} L_{i} = j\bigg ) \nonumber \\
&= \sum_{j = m+1}^{N - m - 1} \binom{ j - 1}{m} (p^*)^ {j - m -1} (1- p^*)^{m+1}. \label{eq:snCDF}
\end{align}

At the end of the $N$-th game, a total of $S_N$ intruders have been able to breach, and therefore, the capture fraction is $\nicefrac{(N - S_N)}{N}$, which is a random variable due to the randomness of $S_N$.
From the (cumulative) distribution function of $S_N$ in \eqref{eq:snCDF}, we compute the expectation $\EX[S_N]$ at the end of the $N$-th game, from which we find the expected percentage of capture at the end of the $N$-th game to be
\begin{align} \label{eq:percentageN}
 {\rm percentage}(N) = \frac{N- \EX[S_N]}{N}\times 100,
\end{align}
for any finite $N$.

\subsection{Percentage of Capture with Infinite Arrival of Intruders}
After the $k$-th travel, the total number of captures is $\sum_{i=1}^k L_i$ and the total number of breach is $k$.
Therefore, the capture fraction $c(k)$ after the $k$-th travel is 
\begin{align*}
    c(k) &=   \frac{\sum_{i=1}^{k} L_i}{\sum_{i=1}^{k} L_i \ + k}.
\end{align*}
Notice that,   $c(k)$ is a random variable for each $k$ since $L_i$'s are random variables. 
For the case of infinite arrival of intruders the number of resets (i.e., $k$) also approaches infinity.
We take the limit $k\to \infty$ to obtain the asymptotic capture fraction $c_\infty$.
Therefore,
\begin{align*}
    c_\infty = \lim_{k \to \infty} c(k) =  \frac{\lim_{k\to\infty}\frac{1}{k}\sum_{i=1}^{k} L_i}{\lim_{k\to\infty}\frac{1}{k}\sum_{i=1}^{k} L_i \ + 1}.
\end{align*}
Although $c(k)$ is a random variable, due to the (strong) law of large numbers we obtain that $\frac{1}{k}\sum_{i=1}^{k} L_i \to \EX[L_i]$ almost surely as $k\to \infty$.
Recall that $L_i$ is a geometric random variable and from \eqref{eq:prob_Li}, we obtain that $\EX[L_i] = \nicefrac{1}{1-p^*}$. 
After substituting $\lim_{k\to\infty}\frac{1}{k}\sum_{i=1}^{k} L_i = \nicefrac{1}{1-p^*}$, we obtain
$
    c_\infty = \nicefrac{1}{(2-p^*)}.
$
Therefore, the asymptotic percentage of capture is
\begin{align} \label{eq:infinitePercentage}
    \mathrm{percentage}(\infty) = \frac{100}{~2-p^*}.
\end{align}
Using the expression of $p^*$ from \eqref{eq:p*}, we may also write $\mathrm{percentage}(\infty) = \nicefrac{100\pi}{(2\pi - \theta_{\max})}$.

\section{Simulation Results} \label{sec:Simu}

We simulate the game with the following parameters $\ro =5$, $\ra = 1,~ \rt = 10$, and $\nu = 0.8 $. 
Under this parametric choice we conducted $100$ random trials of the game.
In each trial we considered a sequence of $200$ incoming intruders that are uniform randomly generated on the TSR boundary. 
The percentage of capture for each trial is plotted in Fig.~\ref{fig:percentageTrial}.
\begin{figure}
    \centering
    \includegraphics[trim = 150 300 180 310, clip, width=0.75 \linewidth]{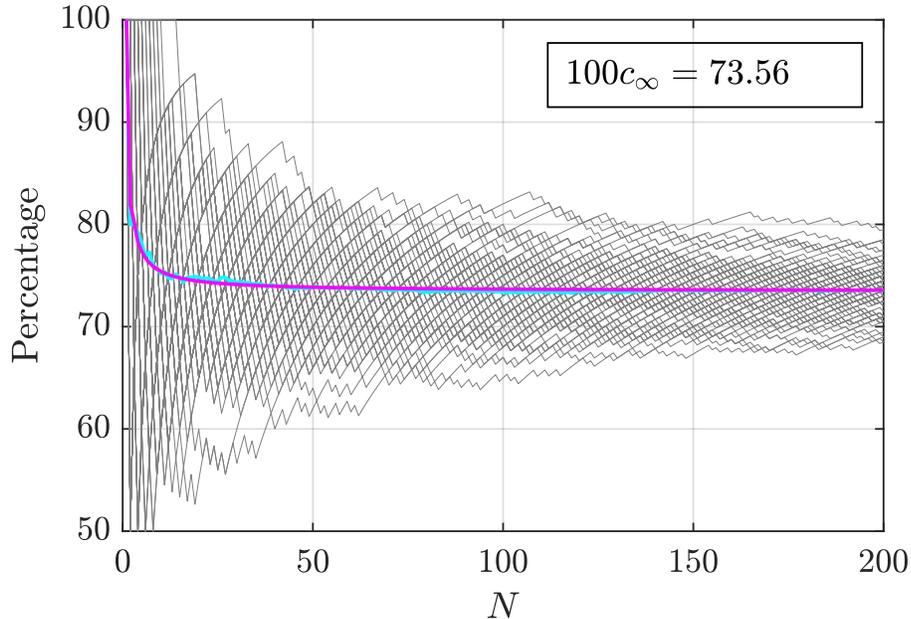}
   
    \caption{Percentage of Capture versus number of arrivals from 100 trials. Each trial is represented with a gray line. 
    Their empirical average is represented by the cyan line.
    The magenta line represents the theoretically predicted capture percentage in \eqref{eq:percentageN}.}
    \label{fig:percentageTrial}
\end{figure}
The abscissa in this figure denotes the number of arrivals ($N$) and the ordinate denotes the percentage of capture (i.e., $100(1-\nicefrac{S_N}N) $) for that number of arrivals.
We compute the empirical mean of the percentage of capture from these 100 random trials, and that is shown by the  cyan line in Fig.~\ref{fig:percentageTrial}.
To compare this simulation result with our theoretical analysis, we plot the expected percentage (i.e., ${\rm{percentage}}(N)$ from \eqref{eq:percentageN}) in the same figure using the 
magenta line. %
We observe that the empirical mean is very close to the theoretically predicted quantity.
In this plot we also report the value of the asymptotic capture percentage (i.e., $100c_\infty$), and we notice that the random trials and $\mathrm{percentage}(N)$ converge ($\sim$~exponentially) to $100c_\infty$ as $N$ increases.
{ A short simulation video can be accessed at \cite{cdcAnimation}. }

\subsection{Parameter variations}
In the next simulation, we vary the parameters $\nu, \ra$ and $\rt$ and plot the asymptotic and some of the finite time expected capture percentages in Figs.~\ref{fig:parametricVariation}-\ref{fig:parametricVariation2}.
\begin{figure*}
    \centering
    \includegraphics[trim = 160 285 225 290, clip, width=0.23\linewidth]{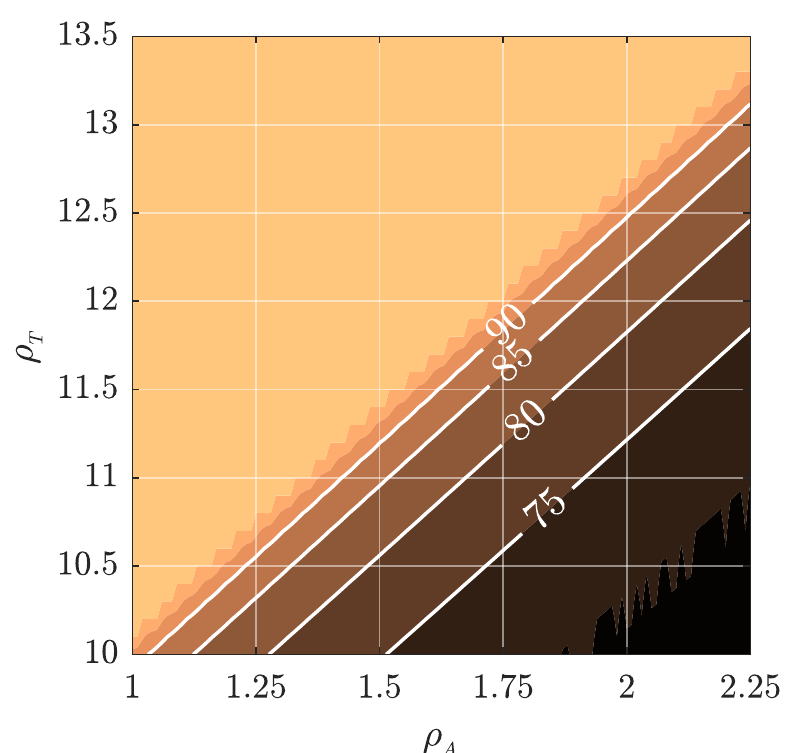}
    \includegraphics[trim = 160 285 225 290, clip, width=0.23\linewidth]{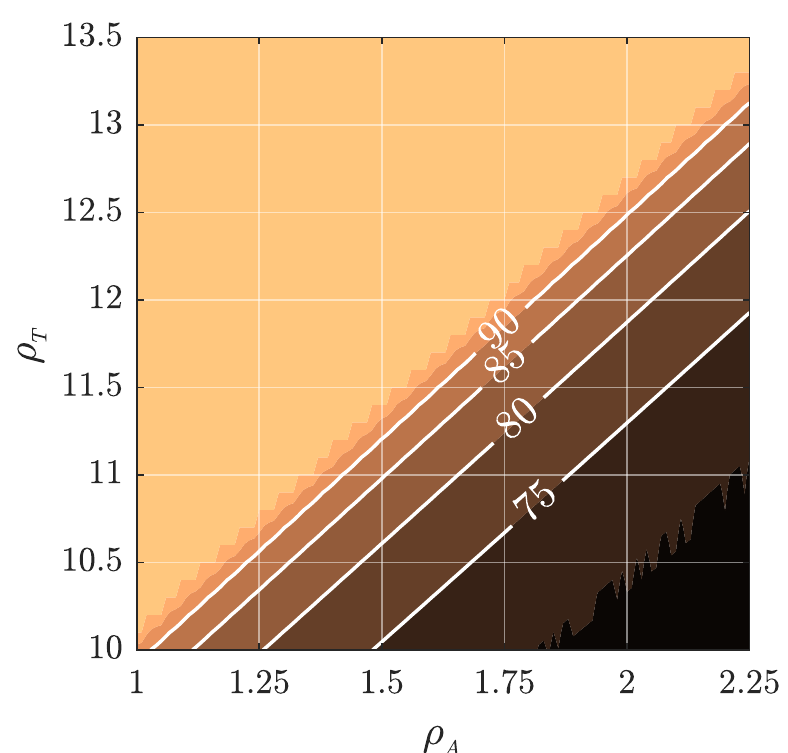}
    \includegraphics[trim = 160 285 225 290, clip, width=0.23\linewidth]{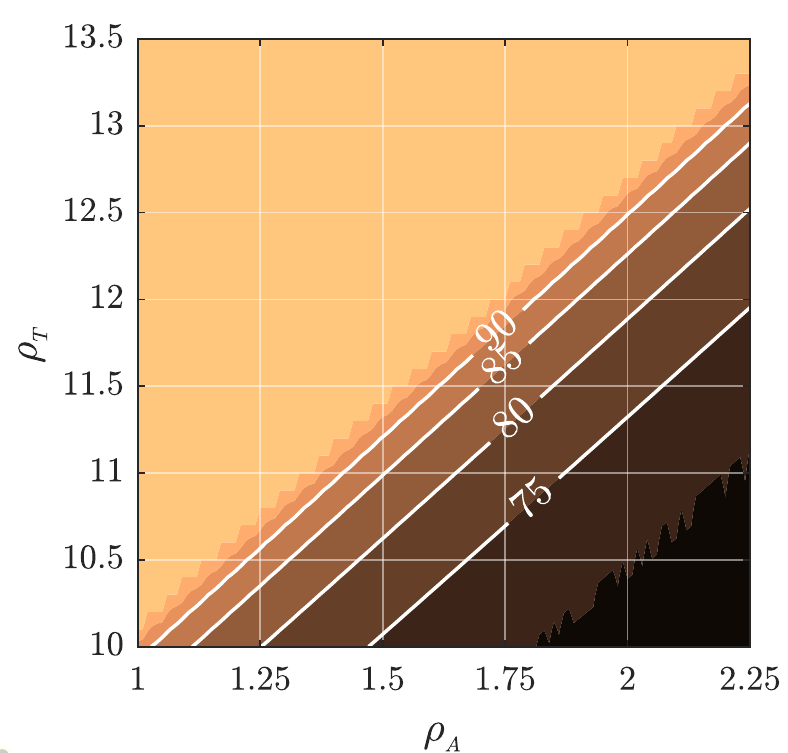}
    \includegraphics[trim = 160 285 185 290, clip, width=0.27\linewidth]{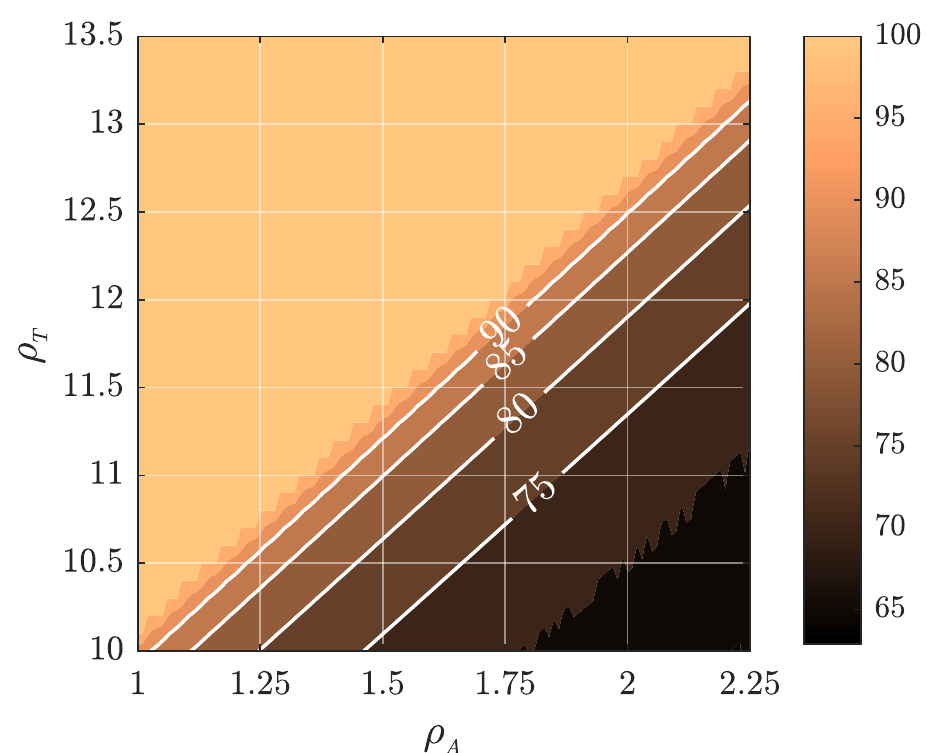}
    \caption{Expected percentage of capture when $\ra$ and $\rt$ are varied for a fixed $\nu = 0.75$. From left to right, we plot the percentage capture for $N = 20, 50, 100$ and $\infty$, respectively.}
    \label{fig:parametricVariation}
\end{figure*}

\begin{figure*}
    \centering
    \includegraphics[trim = 140 255 210 260, clip, width=0.23\linewidth]{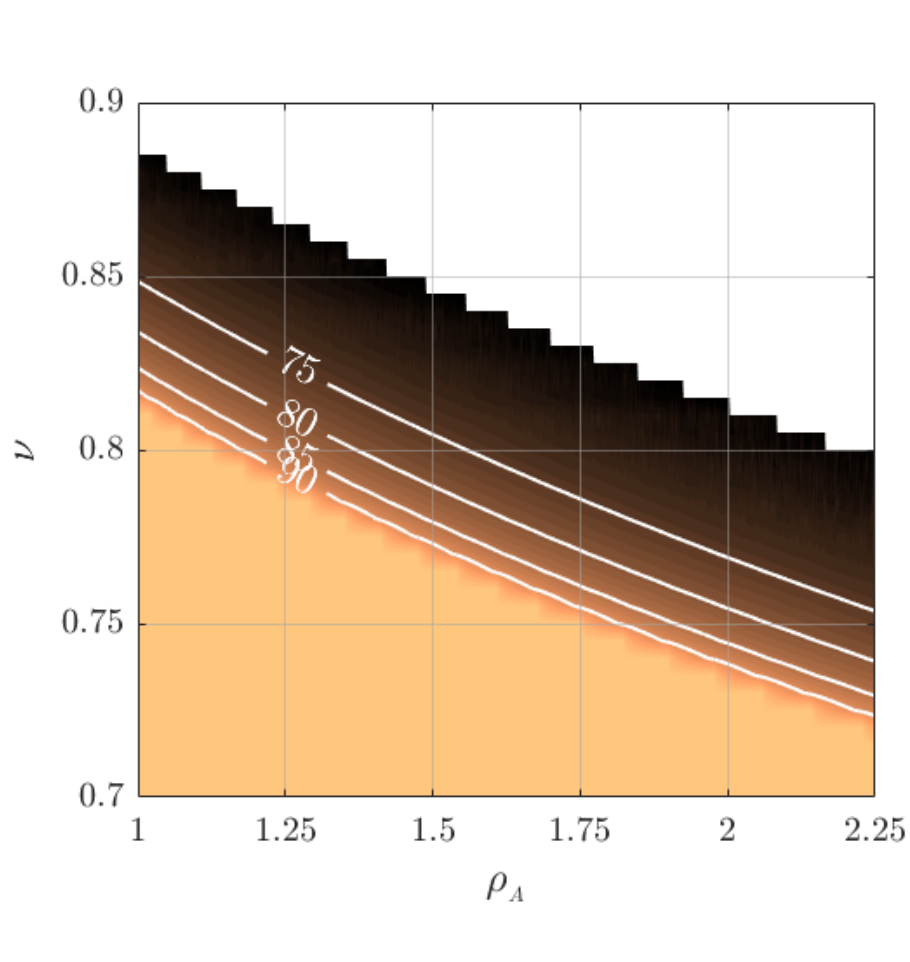}
     \includegraphics[trim = 140 255 210 260, clip, width=0.23\linewidth]{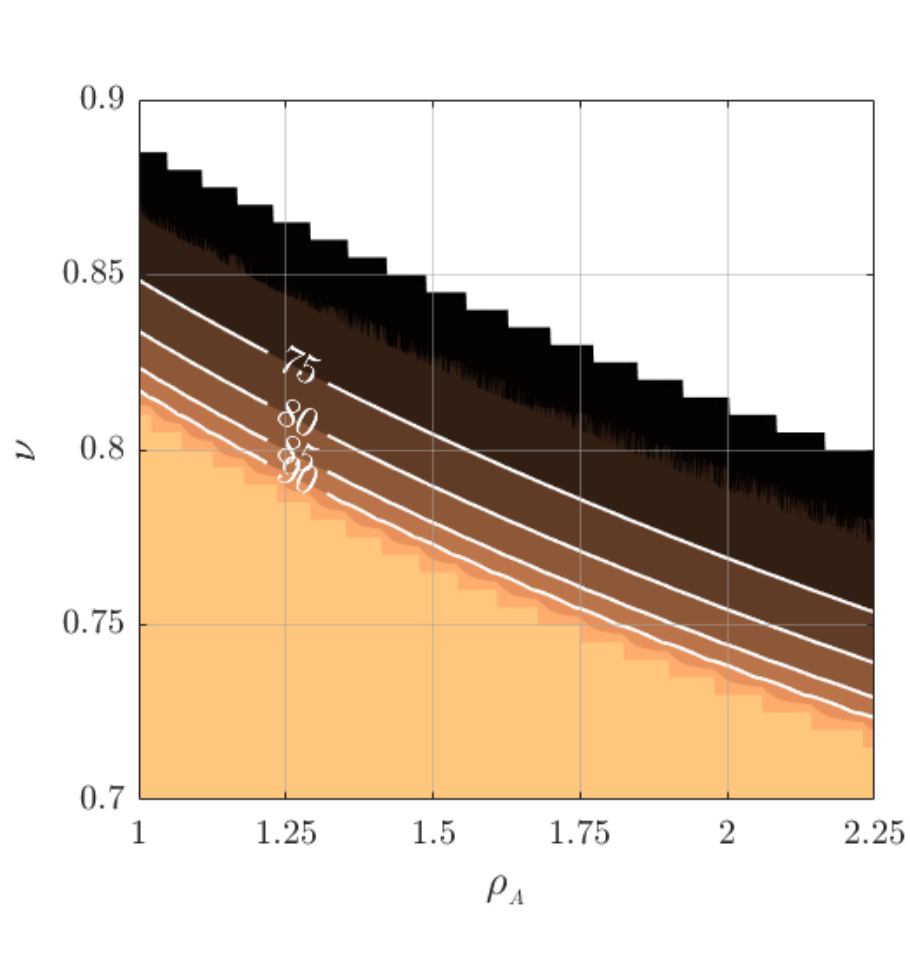}
    \includegraphics[trim = 140 255 210 260, clip, width=0.23\linewidth]{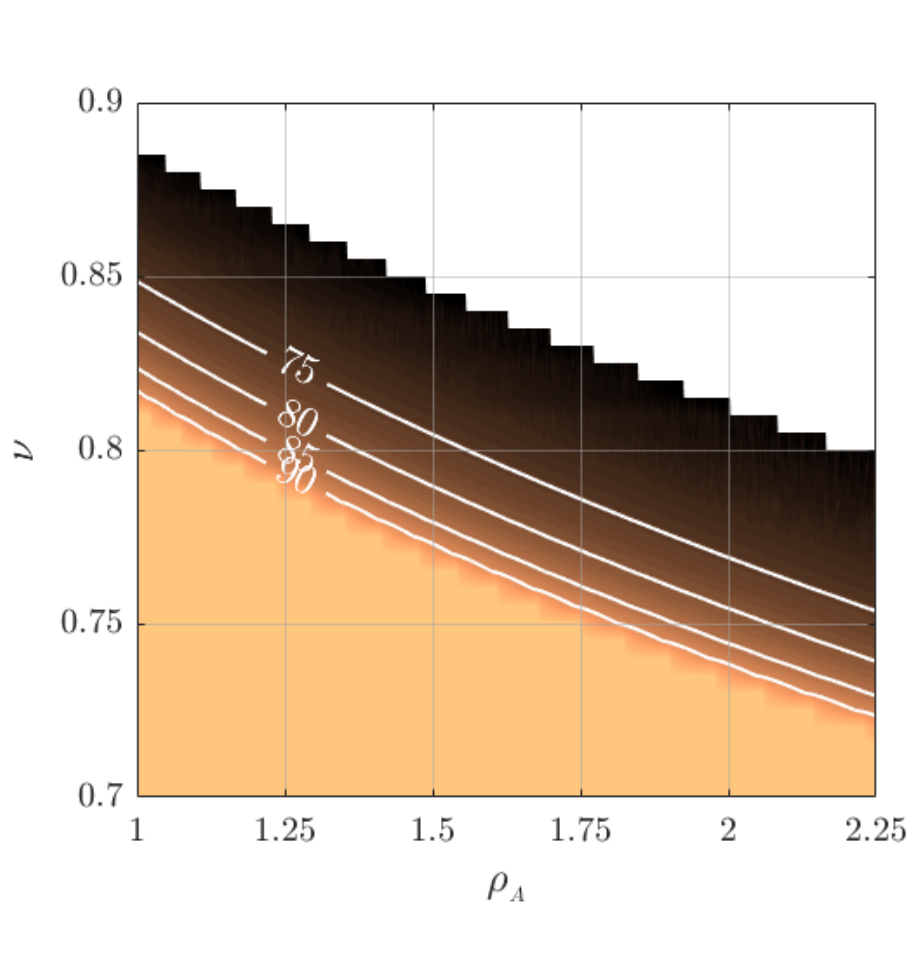}
     \includegraphics[trim = 140 255 170 260, clip, width=0.27\linewidth]{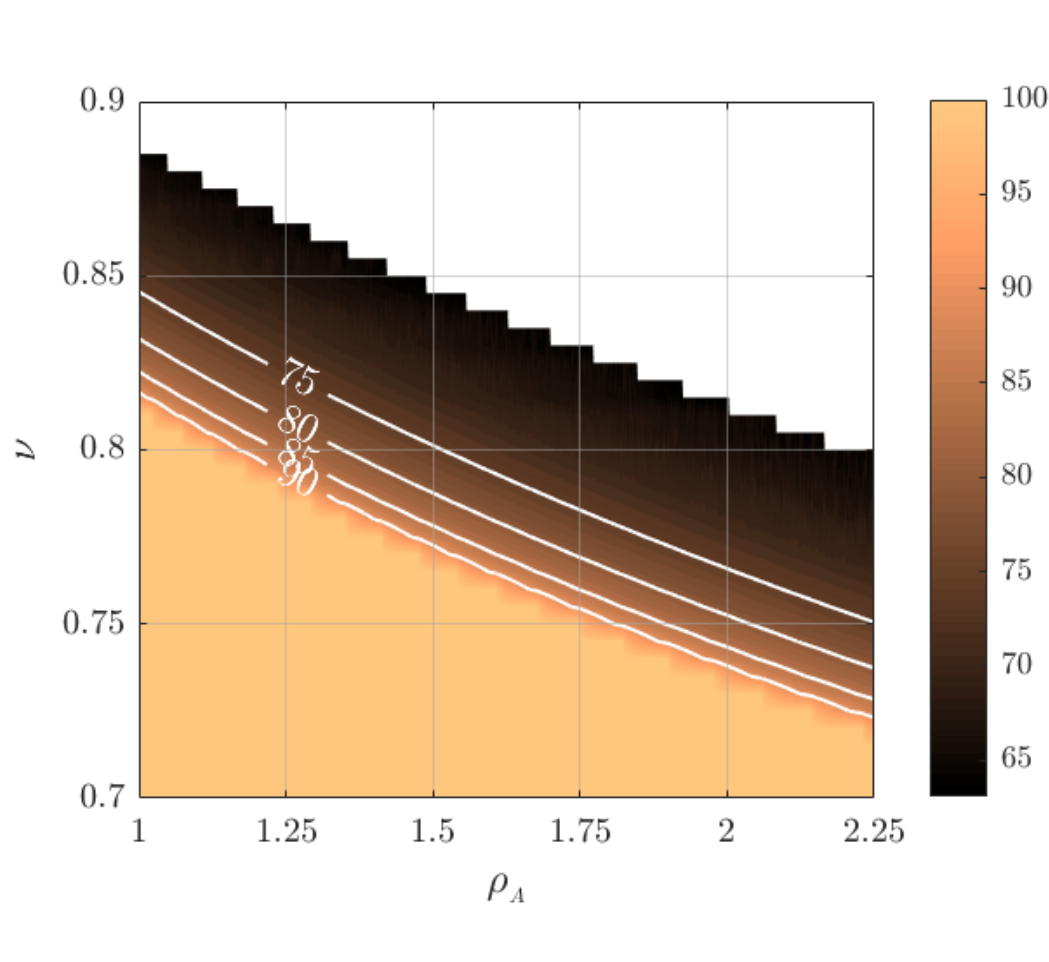}
    \caption{Expected percentage of capture when $\ra$ and $\nu$ are varied for a fixed $\rt = 12$. From left to right, we plot the percentage capture for $N = 5, 20, 50$ and $\infty$, respectively.}
    \label{fig:parametricVariation2}
\end{figure*}
In Fig.~\ref{fig:parametricVariation}, we fix $\nu = 0.75$ and vary the parameters $\ra$ and $\rt$. 
We use \eqref{eq:snCDF}-\eqref{eq:percentageN} to compute $\textrm{percentage}(N)$ for $N=20, 50$ and $100$. 
These are the first three subplots in Fig.~\ref{fig:parametricVariation}.
Then we use \eqref{eq:infinitePercentage} to compute the asymptotic percentage, which is shown in the right-most plot in Fig.~\ref{fig:parametricVariation}.
One of the unexpected observations is to note that many of the level sets (i.e., curves with constant percentages) appear to be linear in the $\ra-\rt$ plane, which is not apparent from the nonlinear relationship between these parameters and $\theta_{\max}$ (or equivalently $p^*$) in Theorem~\ref{thm:angularSeparation}.
Another immediate observation is that the separation between these level sets is \textit{not} proportional to the difference between the percentages. 
In fact, the separation among these lines monotonically decreases even though the percentage is increased at a regular interval. 

For this particular experiment, the slope of these lines appear to be approximately $2.5$. 
An immediate conclusion from this linear behavior is that for each unit of change in $\ra$, we need $2.5$ units of change in $\rt$ to maintain the same expected capture percentage (since the slope is 2.5). 
An interesting future direction would be to investigate these lines and find the relationship between their properties (e.g., slope and intercept) and the problem parameters ($\nu, \ro$ etc.).

Another observation from Fig.~\ref{fig:parametricVariation} is that the plots do not change much as we move from $N=20$ to $N= \infty$. 
This is consistent with our observation in Fig.~\ref{fig:percentageTrial} where the expected capture percentage (the magenta line) quickly converges to the asymptotic capture percentage. 
While Fig.~\ref{fig:percentageTrial} showcases this behavior for one particular parameter setting, Fig.~\ref{fig:parametricVariation} demonstrates the same over a range of parameters.

Next we fixed $\rt$ while varying $\nu$ and $\ra$ to compute the expected capture percentages. 
This is plotted in Fig.~\ref{fig:parametricVariation2}.
When the condition in \eqref{eq:Asm} is not satisfied, we do not compute the capture percentages and hence some parts (top-right) of the plots in Fig.~\ref{fig:parametricVariation2} are empty. 
Some level sets of the capture percentages are also plotted in these plots.
Similar to the plots in Fig.~\ref{fig:parametricVariation}, these plots also do not change a lot as we vary $N$.

\section{Conclusion} \label{sec:Conclusion}
In this paper, we formulated a target defense game against a sequence of incoming intruders. 
At any time only one intruder attempts to breach the target, which decomposes the entire game into a sequence of 1-vs-1 games.
The terminal configuration of the current game becomes the initial configuration for the next game. 
Based on the available sensor information, each game is divided into two phases and the strategies for both agents are derived for these two phases. 
We define the concept of \textit{engagement surface} and \textit{capture circle} to construct the strategies for the agents. 
We analyzed the entire game for both finite and infinite sequences of intruder arrivals and analytically computed the expected percentage of capture for both the cases. 
Numerical experiments demonstrate further interesting and unexpected characteristics in the levels sets (see discussion on Fig.~\ref{fig:parametricVariation}).

A natural extension of this work would be to consider different arrival patterns for the intruders (e.g., periodic arrival, non-uniform probability of arrival locations, multiple arrivals). Furthermore, one may also consider a heterogeneous team for the intruders where different intruders may have different speed and sensing capabilities. 
{{Another possible extension is to consider a defender with its own sensing region.}}

\bibliographystyle{IEEEtran}
\bibliography{reference}

\end{document}